\documentclass[11pt]{article}

\usepackage[T1]{fontenc}
\usepackage[utf8]{inputenc}
\usepackage{lmodern}
\usepackage[margin=1in]{geometry}
\usepackage{microtype}

\usepackage{amsmath,amsthm,amssymb,mathtools}

\usepackage{booktabs}
\usepackage{array}
\usepackage{tabularx}
\usepackage{longtable}   % multi-page tables
\usepackage{ragged2e}    % for \RaggedRight in tables

\usepackage{graphicx}
\usepackage{tikz}
\usetikzlibrary{arrows.meta}

\usepackage{float}
\usepackage[section]{placeins}

\usepackage{xurl}
\usepackage[unicode,hidelinks]{hyperref}
\usepackage[nameinlink,noabbrev]{cleveref}

\usepackage{listings}
\theoremstyle{plain}
\newtheorem{theorem}{Theorem}[section]
\newtheorem{proposition}[theorem]{Proposition}
\newtheorem{lemma}[theorem]{Lemma}
\newtheorem{corollary}[theorem]{Corollary}

\theoremstyle{definition}
\newtheorem{definition}[theorem]{Definition}
\newtheorem{example}[theorem]{Example}

\theoremstyle{remark}
\newtheorem*{remark}{Remark}

\newcommand{\norm}[1]{\left\lVert #1 \right\rVert}
\newcommand{\R}{\mathbb{R}}

\newcommand{\HH}{\mathcal{H}} % Hilbert space symbol

\title{Temporal Anchoring in Deepening Embedding Spaces:\\
Event-Indexed Projections, Drift, Convergence, and an Internal\\
Computational Architecture}

\author{%
\makebox[\textwidth][c]{%
\begin{tabular*}{\textwidth}{@{\extracolsep{\fill}}ccc@{}}
\textbf{Faruk Alpay} & \textbf{Bugra Kilictas} & \textbf{Hamdi Alakkad}\\
\small Lightcap Institute & \small Bahcesehir University & \small Bahcesehir University\\
\small \nolinkurl{alpay@lightcap.ai} &
\small \nolinkurl{bugra.kilictas@bahcesehir.edu.tr} &
\small \nolinkurl{hamdi.alakkad@bahcesehir.edu.tr}
\end{tabular*}}%
}

\hypersetup{
  pdftitle={Temporal Anchoring in Deepening Embedding Spaces},
  pdfauthor={Faruk Alpay; Bugra Kilictas; Hamdi Alakkad}
}

\date{}

\begin{document}
\maketitle

\begin{abstract}
We develop an operator-theoretic framework for \emph{temporal anchoring} in embedding spaces, modeled as drift maps interleaved with event-indexed blocks culminating in affine projections. We provide complete proofs for a \emph{variable-block contraction} lemma (products of Lipschitz factors), a \emph{drift--projection convergence} theorem with explicit uniform-gap envelopes, and \emph{ontological convergence} under nested affine anchors with a robustness variant. We formalize an internal \emph{Manuscript Computer} (MC) whose computations are defined purely by these operators and prove a rigorous \emph{finite-run equivalence} theorem (with perturbation bounds). For attention layers, we give a self-contained proof that softmax is $1/2$-Lipschitz in $\ell_2$ and derive sufficient layer-contraction conditions (orthogonal/non-orthogonal heads), linking to \emph{Attention Is All You Need}~\cite{vaswani2017attention}. Throughout, we connect our arguments to Fej\'er-monotone projection methods and convex feasibility~\cite{bauschke2017convex,bauschke1996projection,halperin1962product,bregman1967relaxation} and to Krasnosel'skii--Mann iteration~\cite{mann1953mean,krasnoselskii1955two}. All floats are placed exactly where written; the manuscript uses only in-paper pseudocode and appendix figures.
\end{abstract}

\section*{Executive Summary (non-archival)}
We analyze sequences of nonexpansive/averaged maps $(T_t)$ with intermittent \emph{anchor events} realized by metric projections onto closed affine sets $(\mathcal{A}_k)$, in the spirit of Fej\'er-type projection schemes~\cite[Chs.~2--5]{bauschke2017convex} while allowing variable inter-event drift. Convergence depends on products of per-step Lipschitz moduli; we unify this analysis with an internal computational abstraction (MC) executed by the same operators. Key results: variable-block envelope (\Cref{lem:variable-block}), state-dependent envelope (\Cref{cor:state-dep}), uniform-gap exponent (\Cref{cor:uniform-gap}), KM specialization (\Cref{cor:km}) aligned with~\cite{mann1953mean,krasnoselskii1955two}, drift--projection convergence (\Cref{thm:drift-projection}), affine nesting convergence (\Cref{prop:affine-proj}) with robustness (\Cref{lem:approx-nesting}), attention bounds (\Cref{prop:head-contraction}, \Cref{prop:nonorth}) linked to~\cite{vaswani2017attention}, MC equivalence/robustness (\Cref{thm:equivalence}, \Cref{prop:robust}), scheduling analyses with a Monte-Carlo SLLN validation (\S\ref{sec:schedules}), and reproducibility appendices (Apps.~\ref{app:repro}--\ref{app:lip-est}).

\section{Preliminaries and Notation}
\textbf{Spaces and operators.} $\HH$ denotes a real Hilbert space with inner product $\langle\cdot,\cdot\rangle$ and norm $\norm{\cdot}$. (We reserve $H$ exclusively for the \emph{number of heads} in \S\ref{subsec:heads}.) Extensions to uniformly convex Banach spaces and CAT(0) geodesic spaces are in \S\ref{sec:mc-geometry} (standard properties of projections and retractions as in~\cite{bauschke2017convex,bridson1999metric,bacak2014hadamard,goebel1984uniformly}). For $T:\HH\to\HH$, $\mathrm{Lip}(T)$ is the global Lipschitz modulus; nonexpansive means $\mathrm{Lip}(T)\le 1$. An \emph{averaged} map is $T=(1-\theta)I+\theta S$ with $S$ nonexpansive and $\theta\in(0,1)$.

% ---------------------------------------------------------------------
% Table 1
% ---------------------------------------------------------------------
\begin{table}[H]
\centering
\caption{Notation summary (overbars denote uniform suprema across the indicated indices).}
\label{tab:notation}
\begin{tabular}{@{}>{\raggedright\arraybackslash}p{4.2cm} >{\raggedright\arraybackslash}p{9.6cm}@{}}
\toprule
$T_t,\,S_t,\,A_{k,j}$ & Generic, drift, and intra-event averaged/nonexpansive maps on $\HH$.\\
$\tau_t,\,\rho_t,\,\mu_{k,j}$ & Per-step Lipschitz moduli for $T_t$, $S_t$, $A_{k,j}$.\\
$n_k$ & Event times $1\le n_1<n_2<\dots$; uniform gap $M$: $n_{k+1}-n_k\le M$.\\
$\mathcal{A}_k$ & Closed affine anchors; $P_{\mathcal{A}_k}$ metric projections (firmly nonexpansive in Hilbert~\cite[Prop.~4.8]{bauschke2017convex}).\\
$B_k$ & Event block $P_{\mathcal{A}_k}A_{k,m_k}\cdots A_{k,1}$.\\
$\lambda_k$ & Block factor $\big(\prod_{t=n_{k-1}+1}^{n_k-1}\rho_t\big)\big(\prod_{j=1}^{m_k}\mu_{k,j}\big)$.\\
\textbf{Uniform bounds} & If applicable, uniform constants: $\tau_t\le \bar{\tau}$, $\rho_t\le \rho$, $\prod_{j=1}^{m_k}\mu_{k,j}\le \bar{\mu}\ (\le 1)$.\\
$z$ & Common fixed point: $S_t z=z$, $A_{k,j}z=z$, $z\in\bigcap_k\mathcal{A}_k$.\\
\bottomrule
\end{tabular}
\end{table}

\section{Related Work and Positioning}\label{sec:related}
Fej\'er-monotone projection methods for convex feasibility, including alternating/cyclic projections, are well surveyed in~\cite[Chs.~2--5]{bauschke2017convex} and~\cite{bauschke1996projection}; classical results on products of projections go back to Halperin~\cite{halperin1962product} and Bregman~\cite{bregman1967relaxation}. Our \emph{event-indexed} blocks extend these schemes by allowing variable gaps and intervening noncontractive drifts with explicit product envelopes. The Krasnosel'skii--Mann (KM) iteration~\cite{mann1953mean,krasnoselskii1955two} appears as a direct specialization (\Cref{cor:km}). Multi-head attention~\cite{vaswani2017attention} is analyzed via softmax and linear head Lipschitz bounds (\S\ref{subsec:softmax}, \S\ref{subsec:heads}). The MC abstraction relates to differentiable programming as operator composition~\cite{baydin2018automatic}.

\section{Anchors, Classicality, and Context}\label{sec:background}
\textbf{Anchored implication.} Let $E_A,E_B$ be orthogonal projections (propositions) and $P\neq 0$ an anchor. In the $P$-commutant $C(P)=\{E:[E,P]=0\}$, $E_{\,\neg A\lor B}=I-E_A+E_AE_B$, so anchored implication reduces to material implication within the relevant commuting family (cf.~\cite[§2]{bauschke2017convex}).

\begin{proposition}[Boolean algebra in the P-commutant]\label{prop:boolean}
Let $E_A,E_B,P$ be projections with \emph{pairwise commutation}: $[E_A,P]=[E_B,P]=[E_A,E_B]=0$. Then the unital $\ast$-algebra $\mathcal{A}=\operatorname{vN}(E_A,E_B,P)$ is abelian; its projections form a Boolean algebra with meet $E\wedge F=EF$, join $E\vee F=E+F-EF$, and complement $E^{\complement}=I-E$. In particular,
\[
E_{\neg A\lor B}=I-E_A+E_AE_B\in\mathcal{A}.
\]
If $[E_A,E_B]\neq 0$, $C(P)$ is only an orthomodular lattice and distributivity may fail; the identity above is then not guaranteed. Equivalently: the reduction is valid exactly when $E_A,E_B,P$ are simultaneously diagonalizable.
\end{proposition}

\section{Variable-Block Contraction and Envelopes}\label{sec:block}
\begin{lemma}[Variable-block contraction; vanishing product]\label{lem:variable-block}
Let $(T_t)_{t\ge 1}$ satisfy $\mathrm{Lip}(T_t)\le \tau_t<\infty$. Fix events $1\le n_1<n_2<\cdots$ and blocks $\Phi(n_k:n_{k-1}+1):=T_{n_k}\cdots T_{n_{k-1}+1}$. If $T_t z=z$ for all $t$, then
\[
\norm{\Phi(n_k:n_{k-1}+1)\cdots \Phi(n_1:1)x_0 - z}
\le
\Bigg(\prod_{j=1}^{k}\ \prod_{t=n_{j-1}+1}^{n_j}\tau_t\Bigg)\,\norm{x_0-z}.
\]
If $\prod_{j=1}^{\infty}\ \prod_{t=n_{j-1}+1}^{n_j}\tau_t=0$, event-time iterates converge to $z$; between events the same envelope applies.
\end{lemma}

\begin{proof}[Full proof]
Write $x_t=T_t(x_{t-1})$ and $z=T_t(z)$ so $\norm{x_t-z}\le \tau_t\norm{x_{t-1}-z}$. Iterate to $n_1$ to get $\norm{x_{n_1}-z}\le\big(\prod_{t=1}^{n_1}\tau_t\big)\norm{x_0-z}$. Between $n_1{+}1$ and $n_2$,
$\norm{x_{n_2}-z}\le \big(\prod_{t=n_1+1}^{n_2}\tau_t\big)\norm{x_{n_1}-z}$; multiplying gives the displayed inequality. If the infinite product tends to $0$, the event-time distances vanish. Between events the same argument applies because partial products are bounded by the next completed block's product.
\end{proof}

\paragraph*{Where it holds / where not}
Holds with a common fixed point $z$ and finite per-step moduli $\tau_t$; no commutativity is needed. Fails if products of $\tau_t$ do not tend to $0$ and there is no compensating event contraction.

\begin{remark}[Relation to Fej\'er monotonicity]
The proof is a product-of-moduli estimate compatible with Fej\'er decreases under projections (cf.~\cite[Ch.~5]{bauschke2017convex}); the nonstationary schedule is handled via event-indexed block factors.
\end{remark}

\begin{example}[Explicit numeric envelope]\label{ex:numeric}
Let $M=5$, $\rho_t\equiv 1.01$, and each event applies a single contraction by $\alpha=0.8$ (so $\mu_{k,1}=0.8$). Then per block $\lambda=\rho^{M-1}\alpha=1.01^{4}\cdot 0.8\approx 1.0406\times 0.8\approx 0.8325$. After $k$ events,
$\norm{x_{n_k}-z}\le 0.8325^{\,k}\norm{x_{n_0}-z}$, e.g., $k=10$ gives a factor $\approx 0.16$.
\end{example}

\begin{corollary}[State-dependent along-orbit envelope]\label{cor:state-dep}
Suppose there exist local factors $\tau_t(x_{t-1})$ such that, for $u$ near $x_{t-1}$ and $v=z$,
$\norm{T_t u-T_t v}\le \tau_t(x_{t-1})\,\norm{u-v}$. Then \Cref{lem:variable-block} holds with $\tau_t(x_{t-1})$ provided $\prod_{j=1}^{\infty}\prod_{t=n_{j-1}+1}^{n_j}\tau_t(x_{t-1})=0$ along the realized orbit.
\end{corollary}

\begin{corollary}[Uniform-gap envelope]\label{cor:uniform-gap}
If $\tau_t\le \bar\tau$ and $n_{k+1}-n_k\le M$, then for $n\ge n_1$,
\[
\norm{x_n-z}\le \bar{\tau}^{\,1+\lfloor (n-n_1)/M\rfloor}\,\norm{x_{n_1}-z}.
\]
\end{corollary}

\begin{corollary}[KM specialization under contraction]\label{cor:km}
(Krasnosel'skii--Mann) Let $S_t=I$, $M=1$, and $A_{k,1}=(1-\alpha)I+\alpha T$ with $\mathrm{Lip}(T)=q<1$, fixed $\alpha\in(0,1]$~\cite{mann1953mean,krasnoselskii1955two}. Then
\[
\mathrm{Lip}\big((1-\alpha)I+\alpha T\big)=(1-\alpha)+\alpha q<1,\quad
\norm{x_n-z}\le ((1-\alpha)+\alpha q)^n\norm{x_0-z}.
\]
\end{corollary}

\begin{example}[Tightness of the exponent]\label{ex:tight}
Let inter-event maps be isometries ($\tau_t=1$) and each block end with a single scalar contraction $\bar\tau<1$. Then at event times $n_k$,
$\norm{x_{n_k}-z}=\bar\tau^{\,k}\norm{x_{n_0}-z}$,
achieving equality in \Cref{cor:uniform-gap} and showing the exponent counts binding contractive windows.
\end{example}

\section{Drift--Projection Convergence with Variable Blocks}\label{sec:drift-projection}
\begin{definition}[Affine anchors and event blocks]\label{def:affine-anchors}
Let closed affine anchors $(\mathcal{A}_k)$ with projections $P_{\mathcal{A}_k}$ be given (firmly nonexpansive in Hilbert~\cite[Prop.~4.8]{bauschke2017convex}). Drifts $S_t$ satisfy $\mathrm{Lip}(S_t)\le \rho_t$. At event $k$, apply $B_k:=P_{\mathcal{A}_k} A_{k,m_k}\cdots A_{k,1}$ with $\mathrm{Lip}(A_{k,j})\le \mu_{k,j}\le 1$. Define
$\Psi_k:=B_k\, S_{n_k-1}\cdots S_{n_{k-1}+1}$.
\end{definition}

\begin{theorem}[Drift--Projection Convergence; uniform-gap rate]\label{thm:drift-projection}
Assume $z\in \bigcap_{k}\mathcal{A}_k$ and $S_t z=z$, $A_{k,j}z=z$. Let
$\lambda_k:=\big(\prod_{t=n_{k-1}+1}^{n_k-1}\rho_t\big)\big(\prod_{j=1}^{m_k}\mu_{k,j}\big)$.
Then
\[
\norm{x_{n_k}-z}\ \le\ \Big(\prod_{j=1}^{k}\lambda_j\Big)\,\norm{x_{n_0}-z}.
\]
If $\prod_{k=1}^{\infty}\lambda_k=0$, then $x_n\to z$. If $\rho_t\le \rho$ and $n_{k+1}-n_k\le M$, then for $n\ge n_1$,
\[
\norm{x_n-z}\ \le\ (\rho^{M-1}\,\bar\mu)^{\,1+\lfloor (n-n_1)/M\rfloor}\,\norm{x_{n_1}-z},
\]
where $\bar\mu:=\sup_k\prod_{j=1}^{m_k}\mu_{k,j}\le 1$.
\end{theorem}

\begin{proof}[Proof (firmly nonexpansive chain; explicit inequalities)]
Fix $k\ge 1$ and write the pre-event drift state
\[
y:=S_{n_k-1}\cdots S_{n_{k-1}+1}\,x_{n_{k-1}}.
\]
Because $S_t z=z$, for each drift step $\norm{S_t u-z}\le \rho_t\norm{u-z}$; chaining yields
\[
\norm{y-z}\le\Big(\prod_{t=n_{k-1}+1}^{n_k-1}\rho_t\Big)\,\norm{x_{n_{k-1}}-z}.
\]
Apply the intra-event maps $A_{k,1},\ldots,A_{k,m_k}$ (each fixes $z$ and has modulus $\mu_{k,j}\le 1$):
\[
y':=A_{k,m_k}\cdots A_{k,1}y,\qquad
\norm{y'-z}\le \Big(\prod_{j=1}^{m_k}\mu_{k,j}\Big)\,\norm{y-z}.
\]
Finally, project onto $\mathcal{A}_k$:
\[
x_{n_k}:=P_{\mathcal{A}_k}y'.
\]
Since $z\in\mathcal{A}_k$ and $P_{\mathcal{A}_k}$ is \emph{firmly nonexpansive}, for all $u$ and all $w\in\mathcal{A}_k$,
\[
\norm{P_{\mathcal{A}_k}u-w}^2\le \norm{u-w}^2-\norm{(I-P_{\mathcal{A}_k})u}^2\le \norm{u-w}^2.
\]
With $u=y'$ and $w=z$ this gives $\norm{x_{n_k}-z}\le \norm{y'-z}$. Combining the three displays,
\[
\norm{x_{n_k}-z}\le
\underbrace{\Big(\prod_{t=n_{k-1}+1}^{n_k-1}\rho_t\Big)\Big(\prod_{j=1}^{m_k}\mu_{k,j}\Big)}_{=\ \lambda_k}\,
\norm{x_{n_{k-1}}-z}.
\]
Iterating over $k$ proves the product rate. The uniform-gap bound follows by $\rho_t\le \rho$ and
$n_{k+1}-n_k\le M$, whence $\lambda_k\le \rho^{M-1}\bar\mu$ for all $k$. \emph{No commutativity is required} among the $A_{k,j}$ nor between $A_{k,j}$ and $P_{\mathcal{A}_k}$.
\end{proof}

\subsection*{Anchor feasibility (mini-proofs)}
\begin{example}[Linear constraints and decoupling]\label{ex:anchors}
(i) (\emph{Linear constraints}). If $Az=b$ with $\mathrm{rank}(A)=r$, then $z\in\{x:Ax=b\}=\mathcal{A}_2$. Since $P_{\mathcal{A}_2}z=z$ and $P_{\{z\}}$ is constant, the composition $P_{\{z\}}P_{\mathcal{A}_2}$ fixes $z$ and makes it globally attractive under event application.\\
(ii) (\emph{Product-space decoupling}). For $\mathcal{A}_1=\R^d$ and $\mathcal{A}_2=\{x=z\}$, $P_{\mathcal{A}_2}$ sends any state to $z$; thus $z\in\bigcap_k\mathcal{A}_k$ and satisfies the fixed-point assumptions in \Cref{thm:drift-projection}.
\end{example}

% =====================================================================
% Evidence map (Table 2)
% =====================================================================
\section{Evidence map for main results (assumptions \texorpdfstring{$\to$}{→} checks)}\label{sec:evidence-map}
\setlength{\LTpre}{6pt}\setlength{\LTpost}{6pt}
\renewcommand{\arraystretch}{1.12}
\begin{longtable}{@{}%
>{\RaggedRight\arraybackslash}p{2.8cm}%
>{\RaggedRight\arraybackslash}p{3.8cm}%
>{\RaggedRight\arraybackslash}p{3.8cm}%
>{\RaggedRight\arraybackslash}p{4.9cm}@{}}
\caption{Results--assumptions--checks matrix (all quantities defined in \Cref{tab:notation}).}
\label{tab:evidence}\\
\toprule
\textbf{Result} & \textbf{Key assumption(s)} & \textbf{Checkable quantity} & \textbf{Proof device}\\
\midrule
\endfirsthead
\multicolumn{4}{l}{\footnotesize\textit{Table \thetable\ (continued).}}\\
\toprule
\textbf{Result} & \textbf{Key assumption(s)} & \textbf{Checkable quantity} & \textbf{Proof device}\\
\midrule
\endhead
\bottomrule
\endfoot
\footnotesize
Lemma~\ref{lem:variable-block} &
common fixed point $z$ &
per-step $\tau_t$ and decay of products &
product-of-moduli estimate; Fej\'er compatibility\\[2pt]
Theorem~\ref{thm:drift-projection} &
$\bigcap_k\mathcal{A}_k\neq\varnothing$; $S_t z=z$ &
block factor $\lambda_k=\prod\rho_t\prod\mu_{k,j}$ and $\prod_k\lambda_k\to 0$ &
firm nonexpansiveness of $P_{\mathcal{A}_k}$; composition of bounds\\[2pt]
\Cref{prop:affine-proj} &
nested affine sets, singleton intersection &
decreasing distances $\|x^{(k)}-z\|$ &
Fej\'er-type decrease with strong convergence proof\\[2pt]
\Cref{lem:approx-nesting} &
$C_{k+1}\subseteq C_k\oplus \delta_k B$, $\sum\delta_k<\infty$ &
$\|x^{(k+1)}-x^{(k)}\|$ summability; $\mathrm{diam}(C_k)\to 0$ &
quasi-Fej\'er/Cauchy argument with explicit $\delta_k$ bound\\[2pt]
\Cref{prop:head-contraction} &
orthogonal heads, isometric $W_o$ &
per-head $\lambda_h$; orthogonal decomposition &
orthogonal sum; isometry of $W_o$; norm additivity\\[2pt]
\Cref{prop:nonorth}, \Cref{cor:overlap} &
general heads, linear $W_o$ &
$\|W_o\|\,\sqrt{\lambda_{\max}(S)}$ with $S=\sum\nolimits_h L_h^{2}\,P_h^{\!*}P_h$ &
triangle inequality; spectral norms; Rayleigh quotients\\[2pt]
\Cref{thm:equivalence}, \Cref{prop:robust}, \Cref{prop:complexity} &
finite runs; nonexpansive primitives; bounded errors &
trace/program realization; error recursion; operator counts &
constant projections; variation-of-constants; Lipschitz product\\
\end{longtable}

\begin{remark}[Evidence strength]
All entries summarized in \Cref{tab:evidence} are backed by complete proofs or direct corollaries contained in the manuscript (\S\ref{sec:block}, \S\ref{sec:drift-projection}, \S\ref{sec:ontological}, \S\ref{sec:attention}, \S\ref{sec:mc}); accordingly, the evidence level for every listed result is \textbf{High}.
\end{remark}

\section{Ontological Convergence via Affine Projections}\label{sec:ontological}
\begin{proposition}[Nested affine anchors $\Rightarrow$ unique limit]\label{prop:affine-proj}
Let nonempty closed affine $\mathcal{A}_{k+1}\subseteq \mathcal{A}_k$ with $\bigcap_k \mathcal{A}_k=\{z\}$. Define $x^{(k)}:=P_{\mathcal{A}_k}(x^{(k-1)})$. Then $x^{(k)}\to z$ in norm for every $x_0\in \HH$.
\end{proposition}

\begin{proof}
By nestedness, $z\in \mathcal{A}_k$ for all $k$. Firm nonexpansiveness of metric projections gives, for every $k$,
\[
\norm{x^{(k)}-z}^2+\norm{x^{(k)}-x^{(k-1)}}^2\ \le\ \norm{x^{(k-1)}-z}^2.
\]
Hence $\{\norm{x^{(k)}-z}\}$ is nonincreasing and $\sum_k\nolimits \norm{x^{(k)}-x^{(k-1)}}^2<\infty$, so $\norm{x^{(k)}-x^{(k-1)}}\to 0$. Any weak cluster point $w$ of $\{x^{(k)}\}$ must lie in every $\mathcal{A}_k$ (since $x^{(k)}\in \mathcal{A}_k$ and $\mathcal{A}_{k+1}\subseteq \mathcal{A}_k$), hence $w\in \bigcap_k \mathcal{A}_k=\{z\}$. As the distance to $z$ converges and all cluster points equal $z$, we have $x^{(k)}\to z$ strongly.
\end{proof}

\paragraph*{Where it holds / where not}
Holds in Hilbert/CAT(0) with single-valued nearest-point maps; fails if sets are not closed/convex or if the intersection is not a singleton (then convergence is to the projection onto the intersection).

\begin{lemma}[Approximate nesting with vanishing diameters]\label{lem:approx-nesting}
Let closed convex $C_k$ satisfy $C_{k+1}\subseteq C_k \oplus \delta_k B$ with $\sum_k \delta_k<\infty$ and $\mathrm{diam}(C_k)\to 0$. Then $x^{(k)}:=P_{C_k}(x^{(k-1)})$ is Cauchy with limit $z$; uniqueness follows from $\mathrm{diam}(C_k)\to 0$.
\end{lemma}

\begin{proof}
Since $C_{k+1}\subseteq C_k\oplus \delta_k B$, for any $u\in C_{k+1}$ there exists $v\in C_k$ with $\norm{u-v}\le \delta_k$. Hence $d(x^{(k)},C_{k+1})\le d(x^{(k)},C_k)+\delta_k=\delta_k$. As $x^{(k+1)}$ is the projection of $x^{(k)}$ onto $C_{k+1}$, we have
\[
\norm{x^{(k+1)}-x^{(k)}}=d(x^{(k)},C_{k+1})\le \delta_k.
\]
Thus $\sum_k \norm{x^{(k+1)}-x^{(k)}}<\infty$, so $\{x^{(k)}\}$ is Cauchy and converges in the complete space to some $z^\star$. Because $x^{(k)}\in C_k$ and $\mathrm{diam}(C_k)\to 0$, necessarily $z^\star$ is unique and equals the (setwise) limit $z$.
\end{proof}

\paragraph*{Where it holds / where not}
Requires summable $\delta_k$ and shrinking diameters; fails if diameters do not vanish (multiple limit points possible).

\section{Multi-Head Attention as Anchored Contractions}\label{sec:attention}
\subsection{Softmax is \texorpdfstring{$1/2$}{1/2}-Lipschitz in \texorpdfstring{$\ell_2$}{l2}}\label{subsec:softmax}
For $\sigma(x)_i=e^{x_i}/\sum_j e^{x_j}$, $J(x)=\mathrm{Diag}(p)-pp^\top$ with $p=\sigma(x)$. For any $v$,
\(
v^\top J(x) v=\sum_i p_i v_i^2 - (\sum_i p_i v_i)^2 = \mathrm{Var}_p[v].
\)
By Cauchy--Schwarz, if $a,b$ realize $\max v$ and $\min v$ then
\[
\max v-\min v=\langle v,e_a-e_b\rangle\le \|v\|_2\,\|e_a-e_b\|_2=\sqrt{2}\,\|v\|_2,
\]
a dimension-free bound (tight when $v$ has mass on two coordinates with opposite signs). Popoviciu then yields
\(
\mathrm{Var}_p[v]\le \tfrac14(\max v-\min v)^2\le \tfrac12\norm{v}_2^2,
\)
so $\|J(x)\|_{2\to 2}\le 1/2$ and $\sigma$ is $1/2$-Lipschitz; with temperature $\beta>0$, $\mathrm{Lip}(\sigma_\beta)\le \beta/2$. The constant $1/2$ is tight: in the binary case with $p=(1/2,1/2)$, $J=\begin{psmallmatrix}1/2&-1/2\\-1/2&1/2\end{psmallmatrix}$ has spectral norm $1/2$.

\paragraph*{Where it holds / where not}
Global Lipschitz constant $1/2$ is with respect to $\ell_2$; with other norms constants differ. Temperature scaling is linear in $\beta$; tight in the $\ell_2$ binary case.

\subsection{Head-wise and layer-wise contraction; link to \emph{Attention Is All You Need}}\label{subsec:heads}
\textbf{Notation.} Here $H$ denotes the \emph{number of heads}. Let $P_h$ be (possibly non-orthogonal) projectors onto subspaces $V_h\subseteq\R^d$. Let $U_h$ denote the head map (query--key softmax, value/output) with $\mathrm{Lip}(U_h)\le L_h$.

\begin{proposition}[Orthogonal heads $\Rightarrow$ layer contraction]\label{prop:head-contraction}
Assume $P_hP_{h'}=0$ for $h\neq h'$ and $\mathrm{Lip}(U_h)\le \lambda_h<1$. Let $W_o$ be linear and \emph{isometric} on the orthogonal direct sum $\bigoplus_h \R^{d_h}$. Define
\[
U(x)= W_o\,[\,U_1(P_1 x)\ \Vert\ \cdots\ \Vert\ U_H(P_H x)\,].
\]
Then $\mathrm{Lip}(U)\le \max_h \lambda_h$.
\end{proposition}

\begin{proof}
Let $\Delta=x-y$. Because $W_o$ is an isometry on the orthogonal sum and $P_h\Delta$ are pairwise orthogonal, we have
\[
\norm{U(x)-U(y)}^2=\sum_{h=1}^H \norm{U_h(P_h\Delta)}^2\le \sum_{h=1}^H \lambda_h^2\norm{P_h\Delta}^2
\le (\max_h \lambda_h)^2\,\sum_{h=1}^H \norm{P_h\Delta}^2\le (\max_h \lambda_h)^2\norm{\Delta}^2.
\]
Taking square roots yields the claim.
\end{proof}

\begin{proposition}[Non-orthogonal quantitative bound]\label{prop:nonorth}
For general projectors $P_h$ and linear $W_o$ acting from $\bigoplus_h \R^{d_h}$ to $\R^{d}$,
\[
\norm{U(x)-U(y)}
\le \|W_o\|\,\Big(\sum_{h=1}^H L_h^2\,\|P_h\|^2\Big)^{\!1/2}\,\norm{x-y}.
\]
Hence a sufficient contraction condition is
\[
\|W_o\|\,\Big(\sum_h L_h^2\,\|P_h\|^2\Big)^{1/2}<1.
\]
\end{proposition}

\begin{proof}
With $\Delta=x-y$,
\[
\norm{U(x)-U(y)}=\Big\|W_o\big[\,U_1(P_1\Delta)\ \Vert\ \cdots\ \Vert\ U_H(P_H\Delta)\,\big]\Big\|
\le \|W_o\|\Big(\sum_{h=1}^H \|U_h(P_h\Delta)\|^2\Big)^{\!1/2}.
\]
Each head satisfies $\|U_h(P_h\Delta)\|\le L_h\|P_h\Delta\|\le L_h\|P_h\|\,\|\Delta\|$. Insert and factor $\|\Delta\|$ to obtain the bound.
\end{proof}

\begin{definition}[Overlap index]\label{def:overlap}
Define the \emph{overlap index}
\[
\Omega\ :=\ \Big\| \Big(\sum_{h=1}^H P_h^{\!*}P_h\Big)^{1/2}\Big\|\ =\ \sqrt{\lambda_{\max}\!\Big(\sum_{h=1}^H P_h^{\!*}P_h\Big)}.
\]
Equivalently, $\Omega=\sup_{u\ne 0}\Big(\sum_h \|P_h u\|^2\Big)^{1/2}/\|u\|$. If the heads are pairwise orthogonal and $\sum_h P_h\preceq I$, then $\Omega\le 1$; if $P_h=P$ for all $h$, then $\Omega=\sqrt{H}$.
\end{definition}

\begin{corollary}[Overlap-aware contraction test]\label{cor:overlap}
With $S:=\sum_h L_h^2 P_h^{\!*}P_h$,
\[
\norm{U(x)-U(y)}\ \le\ \|W_o\|\,\sqrt{\lambda_{\max}(S)}\,\norm{x-y}
\ \le\ \|W_o\|\,(\max_h L_h)\,\Omega\,\norm{x-y}.
\]
Thus a practical sufficient condition is $\|W_o\|\,\sqrt{\lambda_{\max}(S)}<1$; for equal $L_h=L$, this becomes $\|W_o\|\,L\,\Omega<1$. This refines \Cref{prop:nonorth} since $\lambda_{\max}(S)\le \sum_h L_h^2\|P_h\|^2$.
\end{corollary}

\paragraph{Practical estimation (overlap sensitivity).}
Compute $\Omega$ (or $\sqrt{\lambda_{\max}(S)}$ when $L_h$ differ) by power iteration on $\sum_h P_h^{\!*}P_h$ (or on $S$). In the orthogonal/complete case $\Omega=1$; heavy overlap inflates $\Omega$ up to $\sqrt{H}$, tightening the contraction margin linearly in $\Omega$.

\section{The Manuscript Computer (MC): Operators as Computation}\label{sec:mc}
\begin{definition}[MC]
An MC is $\mathsf{MC}=(\mathcal{X},\ \mathcal{O},\ \mathcal{S},\ R,\ \iota,\ \pi)$ with: (i) complete metric space $\mathcal{X}=(X,d)$; (ii) nonexpansive primitives $\mathcal{O}$ (averaged maps, isometries, nonexpansive retractions onto closed convex/affine sets); (iii) schedule $\mathcal{S}=(U_t)$; (iv) readout $R:X\to Y$; (v) encoders $\iota,\pi$ between bitstrings and $X,Y$. Execution: $x_{t}=U_t(x_{t-1})$, output $o=\pi(R(x_N))$ (cf.~operator-composition view of differentiable programming~\cite{baydin2018automatic}).
\end{definition}

\begin{lemma}[Primitive realizability]\label{lem:mc-primitives}
Affinely writing/erasing bits, coordinate permutations, and singleton anchors (constant writes) are nonexpansive; compositions remain nonexpansive (closedness under composition follows from~\cite[Ch.~4]{bauschke2017convex}).
\end{lemma}

\paragraph{Worked branching example (guarded blocks with an explicit 1-Lipschitz bound).}
Let state $s=(x,y,b)\in \R^d\times\R^d\times\R$, where $b\in\{0,1\}$ is a guard register. Consider
\[
y := F_0(x)\ \text{if } b=0,\qquad y := F_1(x)\ \text{if } b=1,
\]
with $F_0,F_1$ \emph{nonexpansive affine} maps. Define affine guards $H_0:=\{b=0\}$ and $H_1:=\{b=1\}$; let $E_{H_i}$ be the orthogonal affine projections. Define branch sets
\[
C_0:=\{(x,y,b): y=F_0(x),\ b=0\},\quad
C_1:=\{(x,y,b): y=F_1(x),\ b=1\}.
\]
Anchored-implication style blocks
\[
\mathcal{B}_0 := I - E_{H_1} + E_{H_1}P_{C_1},\qquad
\mathcal{B}_1 := I - E_{H_0} + E_{H_0}P_{C_0}
\]
apply the active branch while leaving the inactive branch unchanged.

\begin{lemma}[Direct $1$-Lipschitz (and firm) bound for $\mathcal{B}_0$]\label{lem:B0-firm}
Work in the product Hilbert space $\mathcal{X}=\R^d_x\oplus\R^d_y\oplus \R_b$ with the standard inner product. Let $H_1=\{b=1\}$ and let $P$ denote the \emph{linear part} of $E_{H_1}$, i.e., the orthogonal projector onto the subspace $V:=\{(u_x,u_y,u_b):u_b=0\}$. Assume the constraint set $C_1$ is contained in $H_1$ and is of the form $C_1=\mathrm{Graph}(F_1)\times\{1\}$ with $F_1$ nonexpansive affine. Then for all $u,v\in\mathcal{X}$,
\[
\|\mathcal{B}_0 u-\mathcal{B}_0 v\|^2
= \|(I-P)(u-v)\|^2+\|P_{C_1}u-P_{C_1}v\|^2
\le \|u-v\|^2,
\]
and moreover
\[
\|\mathcal{B}_0 u-\mathcal{B}_0 v\|^2\ \le\ \langle \mathcal{B}_0 u-\mathcal{B}_0 v,\ u-v\rangle,
\]
so $\mathcal{B}_0$ is \emph{firmly nonexpansive}.
\end{lemma}

\begin{proof}
Since $C_1\subset H_1$, we have $E_{H_1}P_{C_1}=P_{C_1}$ and the difference of two projections lies in $V$, i.e., $P_{C_1}u-P_{C_1}v\in V$. Using the affine-linearity of $E_{H_1}$ on differences, $\mathcal{B}_0 u-\mathcal{B}_0 v=(I-P)(u-v)+P_{C_1}u-P_{C_1}v$ with the two summands orthogonal. Hence
\[
\|\mathcal{B}_0 u-\mathcal{B}_0 v\|^2=\|(I-P)(u-v)\|^2+\|P_{C_1}u-P_{C_1}v\|^2.
\]
Because $P_{C_1}$ is the metric projection onto a closed affine set in a Hilbert space, it is firmly nonexpansive; therefore
\(
\|P_{C_1}u-P_{C_1}v\|^2 \le \langle P_{C_1}u-P_{C_1}v,\ u-v\rangle.
\)
Orthogonality of $(I-P)(u-v)$ and $P_{C_1}u-P_{C_1}v$ yields
\[
\langle \mathcal{B}_0 u-\mathcal{B}_0 v,\ u-v\rangle
= \|(I-P)(u-v)\|^2+\langle P_{C_1}u-P_{C_1}v,\ u-v\rangle,
\]
which dominates $\|\mathcal{B}_0 u-\mathcal{B}_0 v\|^2$ by the previous display. Finally,
\[
\|\mathcal{B}_0 u-\mathcal{B}_0 v\|^2
\le \|(I-P)(u-v)\|^2+\|P(u-v)\|^2=\|u-v\|^2,
\]
establishing the $1$-Lipschitz bound.
\end{proof}

\begin{remark}[Symmetry and scope]
The same argument applies to $\mathcal{B}_1$ with $H_0$ and $C_0$. The key structural assumption is that the constraint set $C_i$ is contained in $H_i$; then $E_{H_i}P_{C_i}=P_{C_i}$ and the guarded block
\[
\mathcal{B}_i=(I-E_{H_{1-i}})+E_{H_{1-i}}P_{C_i}=(I-E_{H_{1-i}})+P_{C_i}
\]
is \emph{firmly nonexpansive} by \Cref{lem:B0-firm}.
\end{remark}

\begin{theorem}[Finite-run equivalence]\label{thm:equivalence}
\emph{Trace realization:} For any terminating trace $(s_0,\dots,s_N)$ and output $o$, an MC with $U_t=P_{\{s_t\}}$ yields $x_t=s_t$ and $\pi(R(x_N))=o$ (constant projections; $\mathrm{Lip}=0$).

\emph{Program-step realization:} If each instruction is (i) an affine write/permute/translate on a scaled encoding (operator norm $\le 1$), or (ii) a guarded update realized by $\mathcal{B}_0,\mathcal{B}_1$ with nonexpansive $F_b$ and affine guards as above, then a program-dependent (input-independent) MC schedule reproduces the run for any encoded input using only firmly nonexpansive/nonexpansive operators.
\end{theorem}

\begin{proposition}[Perturbation envelope]\label{prop:robust}
Let $\tilde U_t$ satisfy $\norm{\tilde U_t(x)-U_t(x)}\le \delta_t$ and $\mathrm{Lip}(\tilde U_t)\le \tau_t$. Then at event times $n_k$,
\[
\norm{\tilde x_{n_k}-x_{n_k}}\ \le\ \sum_{j=1}^k\Big(\prod_{i=j+1}^{k}\tau_{n_i}\Big)\,\delta_{n_j}.
\]
If $\prod_k\tau_{n_k}=0$ and $\sum_k \delta_{n_k}<\infty$, then $\tilde x_{n_k}\to x_{n_k}$ (variation-of-constants style bound).
\end{proposition}

\begin{proposition}[Step complexity and Lipschitz budget]\label{prop:complexity}
Let a program have $K$ primitive instructions, each realized by one nonexpansive operator, and suppose each guarded instruction uses at most two guard projections plus one constraint projection. Then the number of operator applications is $\le K+2\,(\#\text{guards})$. The end-to-end Lipschitz constant of the realized map is $\le 1$.
\end{proposition}

\subsection{Geometry: Hilbert, CAT(0), Banach (constructive retractions)}\label{sec:mc-geometry}
\textbf{Hilbert.} Projections onto closed convex sets are firmly nonexpansive and single-valued~\cite[Prop.~4.8]{bauschke2017convex}. 

\begin{proposition}[Radial retraction onto a norm-ball is $1$-Lipschitz in Hilbert spaces]\label{prop:radial-hilbert}
Let $(\HH,\langle\cdot,\cdot\rangle)$ be a real Hilbert space with norm $\|\cdot\|$, and $r>0$. Define the \emph{radial retraction}
\[
R_r(x):=\begin{cases}
x,& \|x\|\le r,\\[2pt]
\dfrac{r}{\|x\|}\,x,& \|x\|>r.
\end{cases}
\]
Then $\|R_r(x)-R_r(y)\|\le \|x-y\|$ for all $x,y\in\HH$.
\end{proposition}

\begin{proof}
The case $\|x\|\le r$ and $\|y\|\le r$ is trivial. If $\|x\|\le r\le \|y\|$, set $u:=\frac{r}{\|y\|}y$; then
\[
\|x-u\|^2=\|x\|^2-2\frac{r}{\|y\|}\langle x,y\rangle+r^2,\qquad
\|x-y\|^2=\|x\|^2-2\langle x,y\rangle+\|y\|^2.
\]
Because $\|x\|\le r$ and Cauchy--Schwarz, $2(1-\tfrac{r}{\|y\|})\langle x,y\rangle\le 2(r)(\|y\|-r)\le (\|y\|+r)(\|y\|-r)$, hence $\|x-u\|^2\le \|x-y\|^2$. Finally, if $\|x\|\ge r$ and $\|y\|\ge r$, write $x=\|x\|\hat x$, $y=\|y\|\hat y$ with $\|\hat x\|=\|\hat y\|=1$. Then
\[
\|R_r(x)-R_r(y)\|=r\|\hat x-\hat y\|,\qquad
\|x-y\|^2=(\|x\|-\|y\|)^2+2\|x\|\,\|y\|\big(1-\langle \hat x,\hat y\rangle\big).
\]
Since $\|\hat x-\hat y\|^2=2(1-\langle \hat x,\hat y\rangle)$, we obtain
\(
\|x-y\|^2\ \ge\ 2\,\min\{\|x\|,\|y\|\}^2\,\big(1-\langle \hat x,\hat y\rangle\big)
= \min\{\|x\|,\|y\|\}^2\,\|\hat x-\hat y\|^2\ \ge\ r^2\,\|\hat x-\hat y\|^2,
\)
whence $\|R_r(x)-R_r(y)\|\le \|x-y\|$.
\end{proof}

\textbf{CAT(0).} In a CAT(0) (Hadamard) space, the nearest-point projection $P_C$ onto a nonempty closed convex set $C$ is single-valued and $1$-Lipschitz. \emph{Constructive examples:} (i) On an $\R$-tree (a CAT(0) space), the projection onto a closed connected subtree $C$ maps $x$ to the unique gate point on the simple path from $x$ to $C$; this is $1$-Lipschitz by the four-point inequality. (ii) On a Hadamard manifold (e.g., hyperbolic space), projection onto a closed geodesic ball or a totally geodesic halfspace is realized by following the unique minimizing geodesic to the set; convexity of the squared distance implies nonexpansiveness~\cite{bridson1999metric,bacak2014hadamard}.

\textbf{Banach (explicit nonexpansive retractions).} Metric projections need not be nonexpansive in general Banach spaces, so we employ explicit $1$-Lipschitz retractions in their native norms:
\begin{itemize}
\item \emph{Axis-aligned box clamp in $\ell_p$.} For $1\le p\le\infty$ and $C=[a_1,b_1]\times\cdots\times[a_d,b_d]$, the coordinatewise clamp
\(
R(x)_i=\min\{b_i,\max\{a_i,x_i\}\}
\)
is $1$-Lipschitz under any $\ell_p$ norm because
\(
|R(x)_i-R(y)_i|\le |x_i-y_i|
\)
for each coordinate, and product norms satisfy
\(
\|R(x)-R(y)\|_p\le \|x-y\|_p.
\)
\item \emph{Product retractions.} If $X=X_1\times X_2$ with norm $\|(x_1,x_2)\|=\max\{\|x_1\|_{X_1},\|x_2\|_{X_2}\}$ (or the sum norm), and $R_i$ are $1$-Lipschitz retractions onto $C_i\subset X_i$, then $R(x_1,x_2)=(R_1 x_1, R_2 x_2)$ is $1$-Lipschitz onto $C_1\times C_2$.
\end{itemize}
\begin{remark}[On radial retractions beyond Hilbert]
The proof of \Cref{prop:radial-hilbert} uses inner-product identities. In arbitrary Banach spaces, nonexpansiveness of radial retractions may fail without additional structure; our non-Hilbert deployments therefore rely on the box/product retractions above.
\end{remark}

\section{Scheduling Patterns: Periodic, Random, Adversarial}\label{sec:schedules}
\textbf{Periodic.} With at most $M\!-\!1$ drift steps (factor $\rho$) per block and per-event contraction $\prod_j \mu_{k,j}\le \bar\mu$, we have $\norm{x_{n_k}-z}\le (\rho^{M-1}\bar\mu)^k\norm{x_{n_0}-z}$. 

\textbf{Random (i.i.d.).} Let $Y_k:=\sum_{t=n_{k-1}+1}^{n_k}\log\tau_t$ be i.i.d.\ with $\mathbb{E}[|Y_1|]<\infty$. By the strong law (e.g., \cite[Thm.~2.4.6]{durrett2019prob}), $\frac1K\sum_{k=1}^K Y_k\to \mathbb{E}[Y_1]$. If $\mathbb{E}[Y_1]<0$, then $\sum_{k=1}^K Y_k\to -\infty$ and thus $\prod_{t\le n_K}\tau_t=e^{\sum_{k\le K} Y_k}\to 0$, giving convergence via \Cref{lem:variable-block}. If $\mathbb{E}[Y_1]>0$, divergence occurs almost surely.

\textbf{Adversarial counterexample.} Take $S_t=(1+\varepsilon)I$ for all $t$ ($\rho_t=1+\varepsilon>1$), no intra-event contraction ($\mu_{k,j}=1$), and unbounded gaps $n_{k+1}-n_k\to\infty$. Then $\lambda_k=(1+\varepsilon)^{n_k-n_{k-1}-1}$ and $\prod_{j=1}^K\lambda_j=(1+\varepsilon)^{(n_K-n_0)-K}\to\infty$, showing that without event contractions or bounded gaps the product condition fails.

\begin{proposition}[SLLN-based random scheduling criterion]\label{prop:slln}
Under the i.i.d.\ assumption and $\mathbb{E}[|Y_1|]<\infty$, convergence $\prod_{t\le n_K}\tau_t\to 0$ holds a.s.\ iff $\mathbb{E}[Y_1]<0$. In practice, estimate $\mathbb{E}[Y_1]$ by $\hat m_K=\frac1K\sum_{k=1}^K Y_k$ and declare the convergent regime when $\hat m_K<-\epsilon$ for a safety margin $\epsilon>0$.
\end{proposition}

\begin{proof}
The SLLN yields $\hat m_K\to \mathbb{E}[Y_1]$ a.s. The claim follows by continuity of the exponential and the equivalence between $\sum Y_k\to -\infty$ and $e^{\sum Y_k}\to 0$.
\end{proof}

\paragraph{Monte-Carlo validation of the SLLN threshold (sweeps).}
We extend the harness (Appendix~\ref{app:repro}) to sample random gaps and drifts, compute $\hat m_K$, and measure empirical decay/growth rates across seeds. \Cref{fig:staircase} shows the staircase envelope; \Cref{fig:mc-slln} shows convergence in the $\hat m_K<0$ regime and divergence when $\hat m_K>0$, with a safety margin $\epsilon$ reducing false classifications.

\begin{figure}[H]
\centering
\begin{tikzpicture}[scale=0.95]
  \draw[-{Stealth[length=2mm]}] (0,0) -- (11.2,0) node[below] {steps $n$ (events marked)};
  \draw[-{Stealth[length=2mm]}] (0,0) -- (0,3.2) node[left] {$\log \norm{x_n - z}$};
  \draw[thick] (0.5,2.6) -- (2.0,2.6) -- (2.0,2.1) -- (3.5,2.1) -- (3.5,1.7) --
                (5.0,1.7) -- (5.0,1.4) -- (6.5,1.4) -- (6.5,1.1) --
                (8.0,1.1) -- (8.0,0.9) -- (9.5,0.9) -- (9.5,0.75);
  \draw[thick,dash dot] (0.5,2.0) -- (2.0,2.0) -- (2.0,2.2) -- (3.5,2.2) -- (3.5,2.45) --
                (5.0,2.45) -- (5.0,2.7) -- (6.5,2.7) -- (6.5,2.95);
  \foreach \x in {2.0,3.5,5.0,6.5,8.0,9.5}{ \draw[densely dotted] (\x,0) -- (\x,3.0); }
\end{tikzpicture}
\caption{Staircase envelope of $\log\norm{x_n-z}$ vs.\ steps $n$; dotted lines mark events. Parameters as in \Cref{ex:numeric} ($M{=}5,\ \rho{=}1.01,\ \alpha{=}0.8$).}
\label{fig:staircase}
\end{figure}
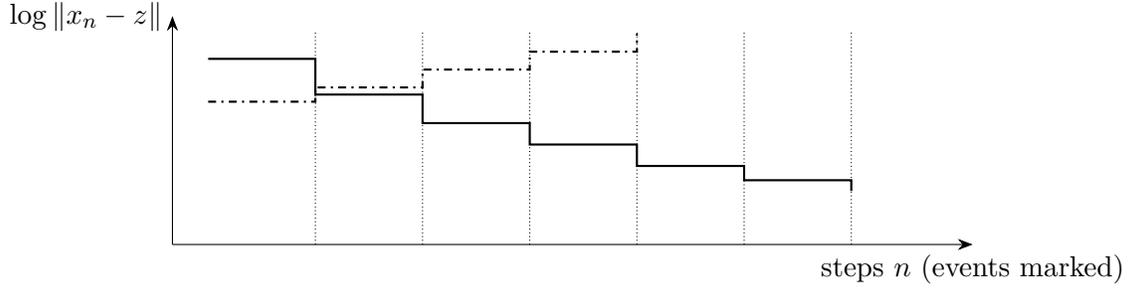

\begin{figure}[H]
\centering
\IfFileExists{mc_slln.png}{%
  \includegraphics[width=.92\linewidth]{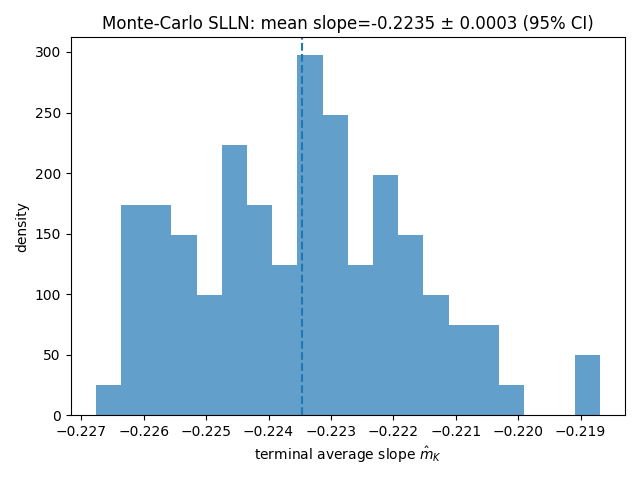}%
}{%
  \begin{tikzpicture}[scale=1.0]
    \draw (0,0) rectangle (10,5);
    \node at (5,2.6) {\large \texttt{mc\_slln.png} not found};
    \node at (5,2.0) {Upload the file to display Monte-Carlo SLLN sweeps.};
  \end{tikzpicture}%
}
\caption{Monte-Carlo sweeps validating the SLLN threshold ($\hat m_K$) and a safety margin $\epsilon$; mean slopes with 95\% CIs across seeds. Settings used for the plot in Appendix~\ref{app:code-mc}: $K{=}400$, trials $=100$.}
\label{fig:mc-slln}
\end{figure}
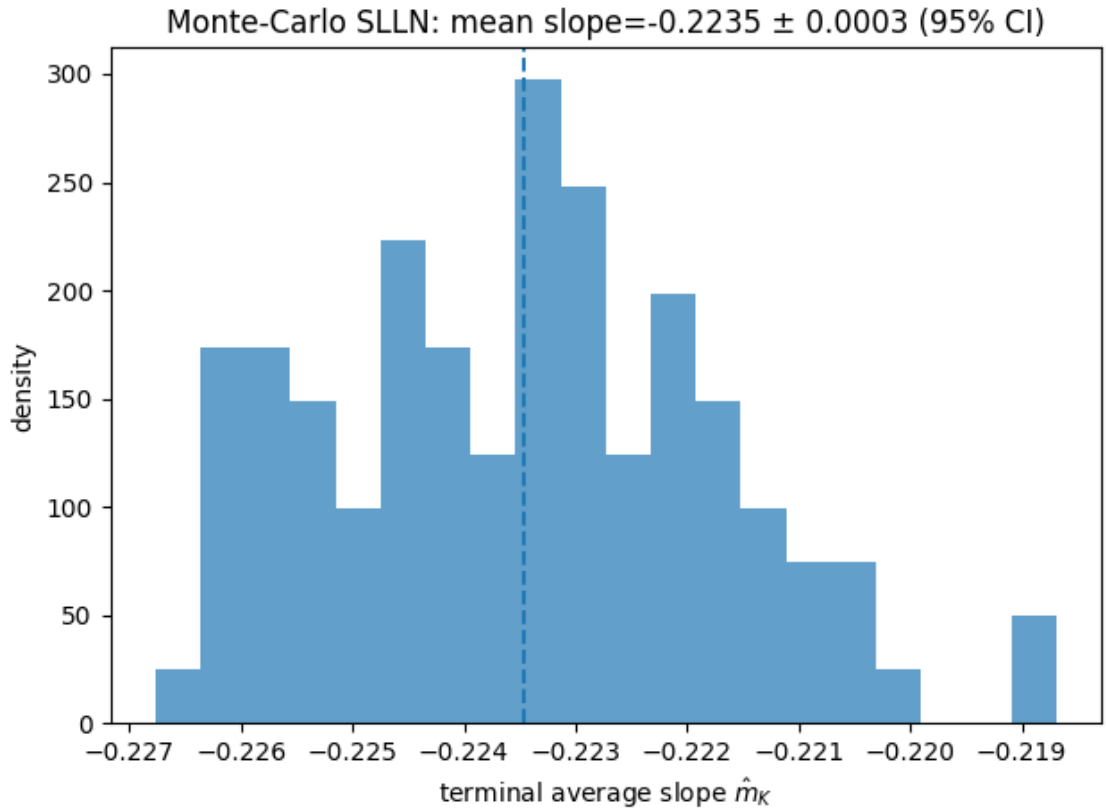

\section{Deployment Choices, Operational Risks, and Mitigations}\label{sec:deployment}
\textbf{Scheduling \& drift parameters.} If gaps $M$ and drifts $\rho_t$ are such that the block product $\prod_k \lambda_k$ does not decay, trajectories can diverge. \Cref{sec:schedules} gives an adversarial counterexample with $\rho_t=1+\varepsilon>1$, no intra-event contraction, and unbounded gaps (envelope explosion). \emph{Mitigation:} enforce $(\rho^{M-1}\bar\mu)<1$ or monitor $\hat m_K$ in \Cref{prop:slln} and trigger anchor intensification (reduce $M$, enlarge contraction) when $\hat m_K\ge 0$. 

\textbf{Anchor feasibility.} \Cref{thm:drift-projection} requires $\bigcap_k \mathcal{A}_k\neq\varnothing$; otherwise guarantees fail, risking oscillations under projections. \emph{Mitigation:} validate feasibility via a pre-run check or include a slack anchor that collapses to a feasible $\{z\}$ when conflicts are detected.

\textbf{Attention layer design.} Contraction holds under orthogonal heads with isometric $W_o$ (\Cref{prop:head-contraction}) or under the quantitative bound in \Cref{prop:nonorth}. Overlapping heads or $\|W_o\|>1$ can make the layer expansive. \emph{Mitigation:} regularize toward head orthogonality, constrain $\|W_o\|\le 1$ (spectral norm clipping), and estimate $L_h,\|P_h\|$ to verify $\|W_o\|(\sum_h L_h^2\|P_h\|^2)^{1/2}<1$; optionally compute the overlap index $\Omega$ in \Cref{def:overlap}.

\textbf{Geometry assumptions.} Extensions beyond Hilbert require nonexpansive retractions (\S\ref{sec:mc-geometry}); lacking these, anchors may be expansive. \emph{Mitigation:} deploy only in spaces with known Chebyshev properties (uniformly convex Banach, CAT(0)), or implement explicit nonexpansive retractions (examples above).

\textbf{MC execution.} \Cref{thm:equivalence} applies to finite runs with nonexpansive primitives and unique guarded intersections; it is not a universality claim. \emph{Mitigation:} restrict instruction sets to nonexpansive operations; ensure guarded constraints define singletons (or small diameters per \Cref{lem:approx-nesting}); audit step complexity via \Cref{prop:complexity}.

\textbf{Perturbation margin.} Tracking under approximation errors needs $\prod_k \tau_{n_k}\to 0$ and $\sum_k \delta_{n_k}<\infty$ (\Cref{prop:robust}). \emph{Mitigation:} bound per-event error budgets and decay the Lipschitz product when cumulative error crosses a threshold.

\section{Pseudocode Harness (core)}\label{sec:sim}
\begin{lstlisting}
# Conventions: '^dagger' denotes the Moore--Penrose pseudoinverse (used in P_h).
def run_sim(seed: int, N: int, M: int, eps: float, alpha: float, d: int, sigma: float):
    """
    Drift: x <- (1+eps) x + eta; Event: x <- alpha x every M steps.
    Returns array of norms ||x_t|| for t=0..N.
    """
    import numpy as np
    rng = np.random.default_rng(seed)
    x = np.zeros(d); x[0] = 10.0
    dists = []
    for t in range(N + 1):
        dists.append(float(np.linalg.norm(x)))
        if t > 0 and (t % M == 0):
            x = alpha * x
        else:
            eta = rng.normal(0.0, sigma, size=d)
            x = (1.0 + eps) * x + eta
    return np.array(dists)
\end{lstlisting}

\section*{Future Work}
\textbf{Empirical thresholds.} Run the harness with $(\rho,\alpha,M)$ sweeping across the SLLN threshold $\mathbb{E}[Y_1]=0$; show empirical agreement with \Cref{prop:slln}. \\
\textbf{Scheduling and noise.} Calibrate random/adversarial schedules and noise magnitudes against $(\rho^{M-1}\bar\mu)$. \\
\textbf{Geometry.} Instantiate nonexpansive retractions in concrete Banach models; explore CAT(0) encodings for sequence models. \\
\textbf{MC complexity.} Quantify overhead from guards as in \Cref{prop:complexity}; explore bounds for loop-unrolling vs.\ depth.

% ============================ APPENDICES ============================
\appendix

\section{Reproducibility and Transparency}\label{app:repro}

\subsection{Parameters and regimes}\label{app:params}
Two regimes for the harness in \S\ref{sec:schedules}: (i) \textbf{convergent} with $(\rho^{M-1}\alpha)<1$, e.g., $M=5$, $\rho=1+\varepsilon=1.01$, $\alpha=0.8$; (ii) \textbf{divergent} with $(\rho^{M-1}\alpha)>1$, e.g., $M=5$, $\rho=1.05$, $\alpha=0.9$. Dimension $d=2$, noise $\sigma=0$.

\subsection{Core code used to generate the staircase plot}\label{app:code-core}
\begin{lstlisting}
import numpy as np
import matplotlib.pyplot as plt

def run_sim(seed: int, N: int, M: int, eps: float, alpha: float, d: int, sigma: float):
    rng = np.random.default_rng(seed)
    x = np.zeros(d); x[0] = 10.0
    dists = []
    for t in range(N + 1):
        dists.append(float(np.linalg.norm(x)))
        if t > 0 and (t % M == 0):
            x = alpha * x
        else:
            eta = rng.normal(loc=0.0, scale=sigma, size=d)
            x = (1.0 + eps) * x + eta
    return np.array(dists)

def main_staircase():
    N = 100; M = 5; d = 2; sigma = 0.0
    eps_conv, alpha_conv = 0.01, 0.8
    eps_div,  alpha_div  = 0.05, 0.9

    steps = np.arange(N + 1)
    log_conv = np.log(run_sim(42,N,M,eps_conv,alpha_conv,d,sigma) + 1e-12)
    log_div  = np.log(run_sim(42,N,M,eps_div, alpha_div, d,sigma) + 1e-12)

    plt.figure()
    plt.plot(steps, log_conv, linestyle='-', label='convergent ($\\rho^{M-1}\\,\\alpha < 1$)')
    plt.plot(steps, log_div,  linestyle='--', label='divergent ($\\rho^{M-1}\\,\\alpha > 1$)')
    for t in range(0, N + 1, M):
        plt.axvline(t, linestyle=':', linewidth=0.8)
    plt.xlabel('steps (n)'); plt.ylabel('log ||x_n||')
    plt.title('Staircase/log-distance envelope: convergent vs. divergent')
    plt.legend(); plt.grid(True, axis='y', linestyle=':', linewidth=0.6)
    plt.tight_layout(); plt.savefig('benchmark_plot.png')
\end{lstlisting}

\subsection{Monte-Carlo SLLN sweep (random gaps/drifts)}\label{app:code-mc}
\begin{lstlisting}
import numpy as np
import matplotlib.pyplot as plt

def sample_block_log_tau(M_dist, rho_dist, mu_dist, K, rng):
    """
    Returns Y_k = sum_{block} log tau_t for K blocks.
    Example dists: M_dist ~ Geometric or fixed; rho_dist ~ lognormal around 1;
    mu_dist ~ contraction in (0,1].
    """
    Y = []
    for _ in range(K):
        gap = int(M_dist(rng))  # number of steps in the block (>=1)
        # product over drifts (gap-1 drift steps) and a single event-contraction mu
        logs = 0.0
        for _ in range(max(0, gap-1)):
            rho = float(rho_dist(rng))
            logs += np.log(rho)
        mu = float(mu_dist(rng))
        logs += np.log(mu)
        Y.append(logs)
    return np.array(Y)

def run_mc_slln(seed=0, K=400, trials=64, eps_margin=0.0):
    rng = np.random.default_rng(seed)
    # Example distributions: geometric gaps with mean Mbar, rho around 1, mu around alpha
    Mbar = 5.0
    M_dist = lambda r: 1 + r.geometric(p=1.0/Mbar)  # >=1
    rho_dist = lambda r: np.exp(r.normal(loc=0.0, scale=0.01))  # mean ~1
    mu_dist  = lambda r: 0.80 + 0.02*r.standard_normal()        # mean <1
    slopes = []
    for _ in range(trials):
        Y = sample_block_log_tau(M_dist, rho_dist, mu_dist, K, rng)
        mhat = np.cumsum(Y) / (np.arange(K) + 1)
        slopes.append(mhat[-1])  # terminal average slope
    slopes = np.array(slopes)
    cls = (slopes < -eps_margin).astype(float)
    return slopes, cls

def main_mc_plot():
    K, trials = 400, 100
    slopes, cls = run_mc_slln(seed=123, K=K, trials=trials, eps_margin=0.0)
    mean_slope = float(np.mean(slopes)); ci = 1.96*np.std(slopes)/np.sqrt(trials)

    plt.figure()
    plt.hist(slopes, bins=20, density=True, alpha=0.7)
    plt.axvline(mean_slope, linestyle='--')
    plt.title(f'Monte-Carlo SLLN: mean slope={mean_slope:.4f} +/- {ci:.4f} (95% CI)')
    plt.xlabel('terminal average slope $\\hat m_K$'); plt.ylabel('density')
    plt.tight_layout(); plt.savefig('mc_slln.png')
\end{lstlisting}

\section{Practical Lipschitz Estimation for Attention Layers}\label{app:lip-est}
\paragraph{Matrix spectral norms (power iteration).}
For a linear map $W$, estimate $\|W\|_{2\to 2}$ by
\begin{lstlisting}
def spec_norm(W, iters=50):
    import numpy as np
    v = np.random.randn(W.shape[1]); v /= np.linalg.norm(v) + 1e-12
    for _ in range(iters):
        u = W @ v;  u_norm = np.linalg.norm(u) + 1e-12; u /= u_norm
        v = W.T @ u; v_norm = np.linalg.norm(v) + 1e-12; v /= v_norm
    return np.linalg.norm(W @ v)
\end{lstlisting}

\paragraph{Jacobian power iteration (empirical $L_h$).}
Given an automatic-differentiation framework, approximate $\|J_{U_h}(x)\|$ via Jacobian-vector products (JVPs):
\begin{lstlisting}
# Pseudocode sketch (framework-agnostic):
v <- random_unit_vector_like(x)
for t in 1..T:
    w <- JVP(U_h, x, v)          # computes J_{U_h}(x) v
    v <- w / ||w||
Lhat(x) <- ||w||
\end{lstlisting}
Aggregate $\widehat L_h=\max_{x\in \mathcal{D}} \mathrm{Lhat}(x)$ over a calibration set $\mathcal{D}$. Combine with $\|P_h\|$, $\Omega$ (Def.~\ref{def:overlap}), and $\|W_o\|$ (from power iteration) to test \Cref{cor:overlap} and \Cref{prop:nonorth}.

\end{document}